%% file: main.tex
\documentclass[letterpaper, 10 pt, conference]{ieeeconf}  

\IEEEoverridecommandlockouts                              

\usepackage{multicol}
\usepackage{multirow}
\usepackage{xcolor}
\usepackage{amsfonts}
\usepackage{amsmath}
\usepackage{graphicx}
\usepackage{algorithmic}
\usepackage{algorithm}
\newtheorem{defn}{Definition}

\input{commands}

\title{Reducing Human-Robot Goal State Divergence with Environment Design}

\author{Kelsey Sikes$^{1}$, Sarah Keren$^{2}$ and Sarath Sreedharan$^{1}$
\thanks{$^{1}$Department of Computer Science, Colorado State University
        }%
\thanks{$^{2}$The Taub Faculty of Computer Science, Technion-Israel Institute of Technology%
}}

\begin{document}

\maketitle

\input{0-abstract}
\input{1-intro}

\input{2-related}
\input{3-background}

\input{4-runnning-example}

\input{5-0-method-main-new}
\input{5-1-compliation}

\input{6-evaluation}
\input{8-conclusion}

\bibliographystyle{IEEEtran}
\bibliography{library}

\end{document}

%% file: commands.tex
\newtheorem{prop}{Proposition}

\newcommand{\Action}[2]  {\ensuremath{\mathit{#1}^{#2}}}
\newcommand{\preP}[1]    {\ensuremath{\Action{pre}{#1}_+}}
\newcommand{\preN}[1]    {\ensuremath{\Action{pre}{#1}_-}}
\newcommand{\add}[1]     {\ensuremath{\Action{add}{#1}}}
\newcommand{\del}[1]     {\ensuremath{\Action{del}{#1}}}

\newcommand{\sharedPreP} {\preP{}}
\newcommand{\sharedPreN} {\preN{}}
\newcommand{\sharedAdd}  {\add{}}
\newcommand{\sharedDel}  {\del{}}

\newcommand{\GSD}{\ensuremath{\mathcal{GSD}}}
\newcommand{\MxGD}{\ensuremath{\mathcal{GD}^{\uparrow}}}

\newcommand{\MnGD}{\ensuremath{\mathcal{GD}^{\downarrow}}}

%% file: 0-abstract.tex
\begin{abstract}
One of the most difficult challenges in creating successful human-AI collaborations is aligning a robot’s behavior with a human user's expectations. When this fails to occur, a robot may misinterpret their specified goals, prompting it to perform actions with unanticipated, potentially dangerous side effects. To avoid this, we propose a new metric we call Goal State Divergence $\mathcal{(GSD)}$, which represents the difference between a robot’s final goal state and the one a human user expected. In cases where $\mathcal{GSD}$ cannot be directly calculated, we show how it can be approximated using maximal and minimal bounds. We then input the $\mathcal{GSD}$ value into our novel human-robot goal alignment (HRGA) design problem, which identifies a minimal set of environment modifications that can prevent mismatches like this. To show the effectiveness of $\mathcal{GSD}$ for reducing differences between human-robot goal states, we empirically evaluate our approach on several standard benchmarks.
\end{abstract}



%% file: 1-intro.tex
\section{Introduction}
As Artificial Intelligence (AI) continues to advance and become a more ubiquitous part of society, human-robot interactions are becoming increasingly common. As a result, designing robots that exhibit behavior that conforms to human expectations is becoming more important than ever. Previous work (cf.  \cite{ZhangPlanExpPred,Chakraborti2017ModelReconciliation}) has shown how addressing expectation mismatches lies at the heart of many human-AI interaction problems. In this paper, we will look at problems that might arise when there are differences between the potential goal states a human user expects a robot to achieve and those it might achieve. 


Specifically, when the user provides a goal specification, they would have some expectation of the exact goal states that might satisfy them. 
However, the behavior the robot may choose in response to such a goal specification may result in a state that differs significantly from what the user expected in the characteristics not strictly provided in their specification. This in turn may result in unanticipated side effects, which in severe cases, could threaten human safety. Such expectation mismatches may arise for diverse reasons, including the human user misunderstanding the robot's state and capabilities or even limitations in their inferential capabilities.
\begin{figure}[h]
    \centering
    \includegraphics[width=.35\textwidth]{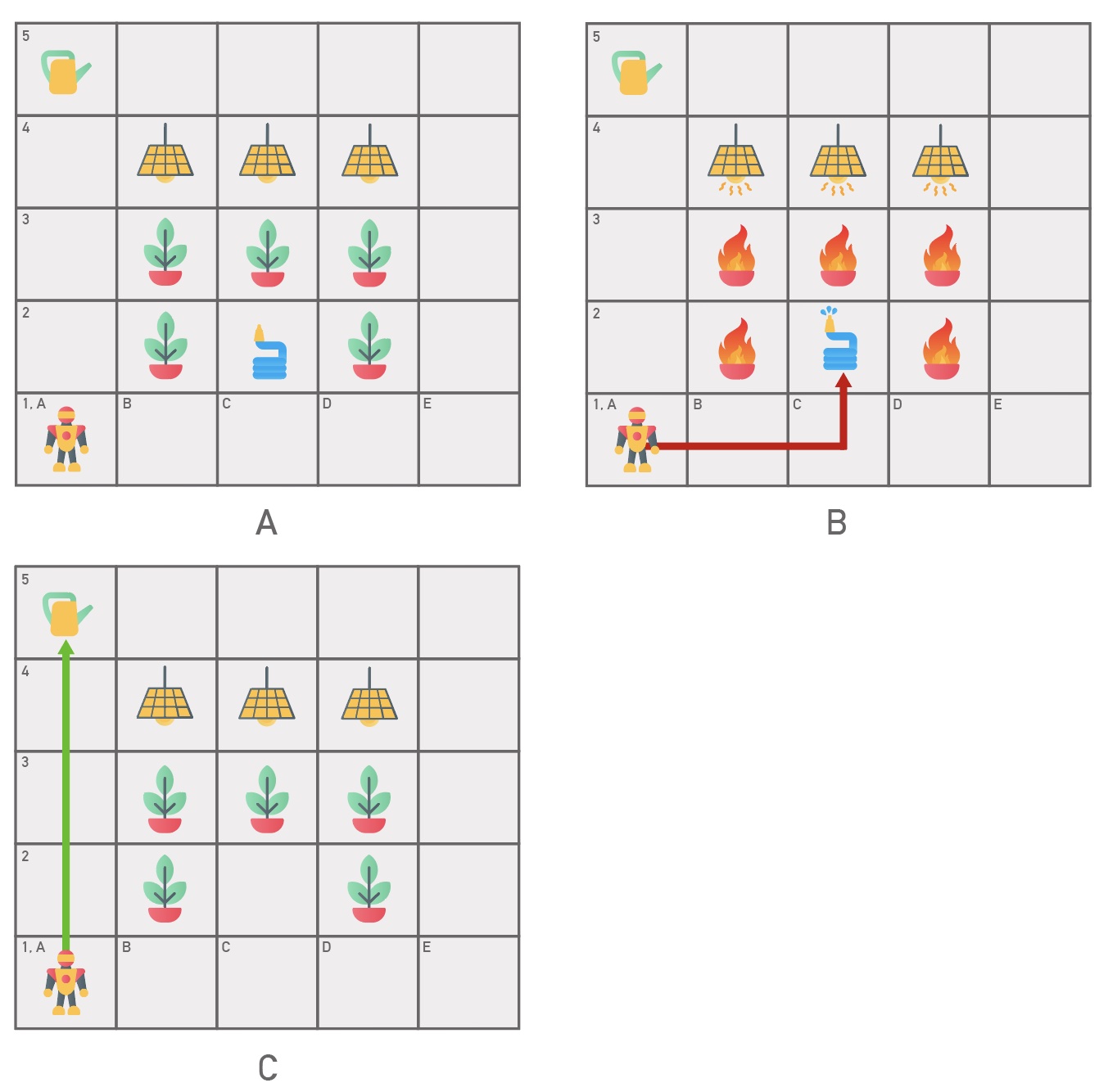}
    \caption{In a greenhouse setting, a human asks a robot to water plants based on their incorrect beliefs about its model. As a result, the robot follows the least costliest plan and chooses to water the plants with a hose, causing a fire. Using environment design, the hose is removed from the scene to avoid potential safety issues.
    }
    \label{fig:fib}
\end{figure}
In this paper, we explore how environment design \cite{ZhangParkes2008} can be used to avoid such potential expectation mismatches. In particular, we look for ways to modify the environment to ensure that the difference between what the human user expects the robot to achieve and what the robot truly achieves is minimized for a given goal specification.
We do so by driving the design process to minimize a novel metric called Goal State Divergence ($\mathcal{GSD}$), which identifies the discrepancy between the final goal state expected by the human and what can be achieved by the robot. 


However, even under generous assumptions about the knowledge the designer has access to, estimating true $\mathcal{GSD}$ presents a unique challenge because the actual human plan may be unknown. In this paper, we instead aim to approximate the magnitude of the divergence through bounds and identify potential environmental modifications that can minimize it. 
Our paper also introduces novel classical planning-based compilations that can identify these bounds for a given design problem. 
To summarize, the primary contributions of this paper are as follows
\begin{enumerate}
    \item We introduce a novel metric to characterize discrepancies between human-robot  goal states in a given planning problem.
    
    \item We develop approximations of the given metric and show how they can be effectively calculated using compilations to classical planning.

    \item We introduce a novel design problem that leverages these approximations to minimize potential final goal state discrepancies.

    \item We present a comprehensive empirical evaluation of our proposed method on standard  benchmarks.
\end{enumerate}

%% file: 2-related.tex
\section{Related Work}
Environment design shapes a robot’s actions by modifying its environment to maximize or minimize some objective \cite{ZhangParkes2008, GRDDetermin}. Several early works in utilizing design in settings where the robots correspond to planning agents have focused primarily on using them to facilitate better goal and plan recognition \cite{KerenGRD2014, MirskyGoalPlanRec}. 
Many of these works have relied primarily on heuristic search methods to identify such designs. \cite{KerenGalKarpasPinedaZilberstein2017} looked at using design to maximize robot objectives in uncertain, stochastic environments, and \cite{KerenEquiUtil} leveraged it to find the maximum shared agent-designer utility in Equi-Reward Utility Maximizing Design (ER-UMD) settings. For the latter, \cite{Keren_Pineda_Gal_Karpas_Zilberstein_2021} extends this work by limiting the space of possible modifications, then mapping each one to a dominating modification. This avoids having to calculate all possible modifications. In our work, we propose a similar, method by identifying a bounded subset of modifications that meets certain criteria regarding the bounds on $\mathcal{GSD}$.

Many works have looked at approaches like value alignment \cite{cirl, MalekValueAlignment} and avoiding side effects \cite{ConcreteProbs, FirstLaw, Gridworlds,klassen2023epistemic,klassen2022planning} as a means of ensuring safe behavior (the implicit assumption being that any behavior that avoids a certain set of states corresponding to the side-effect, will not cause any harm). Many of these works either assume access to a set of locked features or rely on directly querying the user to identify these features (cf. \cite{QuerySideEffects}). Methods like those proposed by \cite{MitNegSideEffects} directly ask humans to update the environment to avoid negative side effects. Unfortunately, this method is hindered by the extensive human intervention it requires. Our method avoids such direct querying by instead relying on a specification of the human model to select an environment modification without requiring any further human intervention. 
We can learn such models by leveraging existing work on learning human mental models (cf. \cite{ModelFreeRec}), in addition to all the works in learning planning models, in general \cite{callanan2022macq}.
Once the human domain knowledge is learned, it can be reused for multiple tasks. Additionally, in many cases, a set of people may share the same model, and we don't necessarily need to learn a unique model for each user \cite{soni2021users}.

Another related area of research is that of explicable planning \cite{ZhangPlanExpPred,KulkarniMinDistExpBeh}, where a robot tries to generate plans aligned with a human’s expectations about what plans the robot may choose.  
 Recently, explicable planning has also been used to mitigate safety issues caused by human-AI model mismatches \cite{hanni2023safe}, where a designer-specified safety bound is used to guarantee that an agent will never select an unsafe behavior. 
Environment design has also been applied to boost the ability of robots to generate explicable planning (cf. \cite{KulkarniDesEnvIntBeh}).
Note that all the previously mentioned explicable planning methods generally focus on matching the human's expectations about the plan as a whole with the final plan carried out by the robot. On the other hand, we solely concentrate on matching the robot’s final goal state with the human’s expectations about potential final goal states and ignore the actual plans that may be used by the robot or expected by the human to achieve them.

%% file: 3-background.tex
\section{Background}
In this section, we define the basic planning terminologies we will be using throughout the paper. 
We define a planning model using the tuple $\mathcal{M} = \langle \mathcal{D}, \mathcal{I}, \mathcal{G}\rangle$. Here $\mathcal{D}$ corresponds to the domain associated with the model and is further defined by the tuple $\mathcal{D} = \langle \mathcal{F, A}, c\rangle$. $\mathcal{F}$ corresponds to the set of propositional fluents that describes the state space corresponding to the given planning problem, such that any state $s$ in that space can be uniquely represented by the set of fluents that are true (i.e. $s \subseteq \mathcal{F}$, for all states $s$).
$\mathcal{A}$ is a set of executable robot actions represented as the tuple $a = \langle \sharedPreP(a), \sharedPreN(a), add(a), del(a)\rangle$. For each action $a \in \mathcal{A}$, $pre_{+/-}(a) \subseteq \mathcal{F}$ are the set of positive or negative preconditions that must be satisfied before $a$ can be executed, while $add(a)$ and $del(a)$ represent sets of add and delete effects for each action $a$. $c$ corresponds to the cost associated with each action.
Finally, $\mathcal{I}$ is the initial state, and $\mathcal{G} \subseteq F$ is the goal specification (which is a partial state specification and not necessarily a state).
We define the effects of executing an action at a given state using a transition function $\mathcal{T}_{\mathcal{M}}: 2^\mathcal{F}\times A \rightarrow 2^\mathcal{F}$, which is given by 
\[ \mathcal{T}_\mathcal{M}(s,a) = \begin{cases}
   s\cup add(a) \ \setminus \ del(a) ~\textrm{ if } exec(s,a) \\
   \textrm{\em{undefined} } ~~~~~~~~~~~~~~~~~\textrm{otherwise}
\end{cases}\]
where $exec(s,a)$ returns true if $s \supseteq \sharedPreP(a)$ and $s \not\supseteq \sharedPreN(a)$. 
We will overload the notation and use the transition function to also be applicable to action sequences, such that $\mathcal{T}_{\mathcal{M}}(s, \langle a_1,...,a_k\rangle) = \mathcal{T}_{\mathcal{M}}(...(\mathcal{T}_{\mathcal{M}}(s,a_1),......,a_k))$.


A solution to a planning problem is a plan, which is an action sequence whose execution in the initial state results in a state that satisfies the goal specification, i.e., $\pi$ is a plan if $\mathcal{T}_{\mathcal{M}}(\mathcal{I},\pi) \supseteq \mathcal{G}$.
We will refer to a state that satisfies the goal specification as a goal state.
Each action in this plan has a cost; summing these reveals the cost of a plan, denoted by $c(\pi) = \sum_{a_i \in \pi} c(a_i)$. A plan is considered optimal if there exists no other plan with a lower cost, and we will represent the set of optimal plans for a model $\mathcal{M}$, with the notation $\Pi^*_\mathcal{M}$ and use $\Pi_\mathcal{M}$ to denote the set of all plans.

%% file: 4-runnning-example.tex
\section{Running Example}
Consider a robot operating in a greenhouse, tasked with completing various chores to maintain its operation. Here, a human assigns tasks to the robot based on their beliefs about its current state and capabilities. The robot then seeks to accomplish these tasks by following a plan that it believes will achieve the specified objective. In the best-case scenario, this plan may result in a goal state which is perfectly aligned with what the human was expecting or just lead to a few minor inconveniences. However, in the worst-case scenario, the robot could carry out a plan with potentially dangerous effects that the human did not anticipate.

As a specific example, consider a scenario where a human asks a robot to water a section of plants, sitting under a series of heat lamps. In providing this goal specification, the human expects the robot to carefully use a watering pail to water the plants, ensuring they remain adequately hydrated. Instead, the robot grabs a nearby hose and haphazardly sprays the plants, splashing water all over --- including onto the heat lamps. This sudden change in temperature causes thermal shock, resulting in the heat lamps shattering and releasing sparks onto the plants below. Moments later, the plants ignite, setting the greenhouse on fire. The human who was not aware that the robot could use hoses completely overlooked this possibility.


To avoid situations like this, environment design can be a useful tool for influencing a robot's decision-making. In the greenhouse setting, for example, the robot's choice of what tool to water the plants with could have been dictated by their placement. Here, the watering pail could have been placed closer to the robot --- increasing the possibility it'd be used --- whereas the hose could have been placed farther away or left in an inaccessible position. Additionally, the heat lamps could have been outfitted with protective covers to ensure they remained shielded from any direct contact with water while the plants were being cared for. Our objective through this paper will be to design algorithms that can automatically identify such potential designs and limit potential mismatches.

%% file: 5-0-method-main-new.tex
\section{Design to Reduce Goal State Divergence}
As discussed, the mismatch between the human's perception of the robot's capabilities/state and the reality could lead to users misspecifying their objective, potentially resulting in unanticipated outcomes. 
To develop methods that can account for and avoid such unintended outcomes, we first need to develop metrics to quantify the degree of mismatch. Specifically, we will start by looking at pairs of states.

\begin{defn}
\textit{Given any two states, $s^1, {s}^2$, state divergence (SD) is defined as the symmetric difference\footnote{The symmetric difference between two states is the number of elements present in either state but not both, which we denote using $\Delta$.} between their respective fluents, i.e.:}
\begin{align*} 
\mathcal{SD}(s^1, s^2) = s^1 \Delta \  s^2 
\end{align*}
\end{defn}

In this paper, we are not interested in just measuring the difference between two arbitrary states but rather the goal state expected by the human and the goal state that the robot can achieve. One could make the case that the intermediate states the robot passes through are as important as the final goal state for many safety applications. However, it is important to note that a purely goal-based specification is general enough to account for such considerations easily. We can introduce new fluents that track intermediate states, and their value in the goal state can be used to account for whether the robot visited any undesirable intermediate state. This requires us to measure the difference in states achievable across models:

\begin{defn}
\textit{For a pair of models that are not necessarily distinct, $\mathcal{M}^1$ and $\mathcal{M}^2$, let $\pi^1$ be a valid plan in $\mathcal{M}^1$, and $\pi^2$ be a valid plan in $\mathcal{M}^2$. Given this, goal state divergence ($\GSD$) of the plan-model pairs is defined as the state divergence between the final state of these two plans, i.e.:}
\begin{align*} 
\GSD(\pi^1, \mathcal{M}^1, \pi^2, \mathcal{M}^2) = \mathcal{SD} ( \mathcal{T}_{\mathcal{M}^1}(\mathcal{I}^1, \pi^1) \ , \ \mathcal{T}_{\mathcal{M}^2}(\mathcal{I}^2, \pi^2) ) 
\end{align*} 
\end{defn}

In our setting, these two models correspond to the robot model $\mathcal{M^R} = \langle \mathcal{D^R}, \mathcal{I^R}, \mathcal{G^R}\rangle$, and the human's model of the robot and the task $\mathcal{M^H}= \langle \mathcal{D^H}, \mathcal{I^H}, \mathcal{G^H}\rangle$. 
We are specifically looking at cases where the robot is trying to follow the goal specification provided by the human exactly, and thus, we have $\mathcal{G^H} = \mathcal{G^R}$.
To simplify the notations, we will also assume that the human and robot models share the same fluent set $\mathcal{F}$. Let $\pi^\mathcal{R}$ be the robot plan and $\pi^\mathcal{H}$ be the plan expected by the human. The central metric of interest for this paper then becomes $\GSD(\pi^\mathcal{H}, \mathcal{M^H}, \pi^\mathcal{R}, \mathcal{M^R})$.

 Note that calculating the above difference requires the system to have access to the human model $\mathcal{M^H}$ and plan $\pi^\mathcal{H}$.  As discussed, there are model learning methods we could employ to learn $\mathcal{M^H}$; Additionally, we will assume that the human's model is given because under many structured settings, the human model may be known beforehand. In the greenhouse case, if the human has been working with a previous model of the robot, their beliefs about the robot would be heavily influenced by the model's capabilities. 
 
 It is worth noting that access to a human model doesn't mean that the robot could potentially avoid goal state divergence by executing plans that are valid in both models since such a plan might not exist.
 Coming to $\pi^\mathcal{H}$, even with a known $\mathcal{M^H}$, the exact plan the human chooses may not be known beforehand because multiple plans may satisfy a goal state, any of which the human could choose. 



In cases in which it is not possible to compute $\GSD$ exactly, we will instead consider approximations. The first approximation we will consider is the worst-case approximation, where we will look at the maximum divergence possible between an expected human goal state and what the robot can achieve, more formally,
\begin{defn}
\textit{For two given models, $\mathcal{M}^1$, $\mathcal{M}^2$, the worst-case or maximal goal state divergence ($\MxGD$) is given by the cardinality of the maximum goal state divergence possible between all executable plans in $\mathcal{M}^1$, $\Pi_{{\mathcal{M}^1}}$, and $\mathcal{M}^2$, i.e.:}
\begin{align*}
\MxGD(\mathcal{M}^1, \mathcal{M}^2) = \max_{ \pi^1 \in \Pi_{{\mathcal{M}^1}}, \pi^2 \in \Pi_{{\mathcal{M}^2}} }(|\GSD(\pi^1, \mathcal{M}^1, \pi^2, \mathcal{M}^2)|) 
\end{align*}
\end{defn}

This brings us to our first proposition which states that $\MxGD$ is guaranteed to be an upper bound of the true goal state divergence.
\begin{prop}
\label{prop:upper}
    \textit{For the robot and human model pair $\mathcal{M^R}$ and $\mathcal{M^H}$, the maximal goal state divergence is guaranteed to be greater than or equal to the goal state divergence for the human plan $\pi^\mathcal{H}$ and the robot plan $\pi^\mathcal{R}$, i.e.,  $\MxGD(\mathcal{M^H}, \mathcal{M^R}) \geq |\GSD(\pi^\mathcal{H}, \mathcal{M^H}, \pi^\mathcal{R}, \mathcal{M^R})|$.}
\end{prop}

The validity of the above proposition can be trivially proven from the definition of $\MxGD$. From the proposition, we can assert that one way to reduce goal state divergence is to reduce $\MxGD$. Especially, if we can reduce $\MxGD$ to zero, we are guaranteed that $\GSD$, will be an empty set. 

However, $\MxGD$ could be a loose upper bound, and reducing $\MxGD$ will not necessarily always reduce $\GSD$. Another approximation we could use is the lower bound on $\GSD$. We define this measure similar to $\MxGD$, but now focusing on minimizing the divergence.

\begin{defn}
\textit{For two given models, $\mathcal{M}^1$, $\mathcal{M}^2$, the best-case or minimal goal state divergence ($\MnGD$) is given by the cardinality of the minimum goal state divergence possible between all executable plans in $\mathcal{M}^1$, $\Pi_{{\mathcal{M}^1}}$, and $\mathcal{M}^2$, i.e.:}
\begin{align*}
\MnGD(\mathcal{M}^1, \mathcal{M}^2) = \min_{ \pi^1 \in \Pi_{{\mathcal{M}^1}}, \pi^2 \in \Pi_{{\mathcal{M}^2}} }(|\GSD(\pi^1, \mathcal{M}^1, \pi^2, \mathcal{M}^2)|) 
\end{align*}
\end{defn}

Similar to Proposition \ref{prop:upper}, we can assert that $\MnGD$ provides a lower bound.
\begin{prop}
\label{prop:lower}
    \textit{For the robot and human model pair $\mathcal{M^R}$ and $\mathcal{M^H}$, the minimal goal state divergence is guaranteed to be less than or equal to the goal state divergence for the human plan $\pi^\mathcal{H}$ and the robot plan $\pi^\mathcal{R}$, i.e.,  $\MnGD(\mathcal{M^H}, \mathcal{M^R}) \leq |\GSD(\pi^\mathcal{H}, \mathcal{M^H}, \pi^\mathcal{R}, \mathcal{M^R})|$.}
\end{prop}

We can characterize our unknown $\GSD$, using this upper bound, lower bound, and the gap between the two (i.e., $\MxGD(\mathcal{M^H}, \mathcal{M^R}) - \MnGD(\mathcal{M^H}, \mathcal{M^R})$). 



Now with the basic setting and metrics in place, we are ready to finally define the design problem:

\begin{defn}
\textit{A \textbf{human-robot goal-state alignment (HRGA)} design problem is characterized by the tuple, $\mathcal{DP} = \langle \mathcal{M^R}, \mathcal{M^H}, \mathbb{U}, \Lambda, \mathcal{C} \rangle$, where:}

\begin{itemize}
\item \textit{$\mathcal{M^R}$, $\mathcal{M^H}$, are the initial robot and human models.}
\item \textit{$\mathbb{U}$ is a set of available environment modifications or model updates. These may include changes to the state space, action preconditions, action effects, action costs, initial state, or goal.}
\item \textit{$\Lambda : \mathbb{M} \times \mathbb{U} \rightarrow \mathbb{M}$ is the transition function over a space of possible models. The function generates the model that would be obtained by performing the set of modifications on a given model.}
\item \textit{$\mathcal{C}$ is an additive cost function that maps each design modification in $\mathbb{U}$ to a cost.} 

\end{itemize}
\end{defn}

In the running example, model designs (we use the term design and environment modification interchangeably) may include moving objects like the watering pail or hose around the environment or removing them completely.

One could define various classes of solutions based on the metrics we have described above. The most basic one aims to minimize the design cost while requiring the lower bound and upper bound to fall below a specific threshold.
\begin{defn}
\label{def:kl}
\textit{A  $(l, k)$-bounded minimal solution to a $\mathcal{DP}$ is the cheapest subset of modifications\footnote{Note that this definition makes an implicit assumption that each design is independent and as such can be performed in any order.} $\xi$ that satisfies the following conditions:}
\begin{align*} 
\xi^* = \operatorname*{argmin}_{ \xi \in 2^\mathbb{U}} \mathcal{C}(\xi)
\end{align*}
\begin{align*} 
\textrm{\textit{Such that} }\MnGD(\mathcal{M_{\xi}^{R}}, \mathcal{M_\xi^{H}})\leq \ell \textrm{ \textit{and} }\MxGD(\mathcal{M}_{\xi}^{\mathcal{R}},\mathcal{M}_{\xi}^{\mathcal{H}})\leq k
\end{align*}
\textit{\textit{Where} $\mathcal{M_{\xi}^{R}} = \Lambda(\mathcal{M^{R}}, \xi)$ \textit{and} $\mathcal{M_{\xi}^{H}} = \Lambda(\mathcal{M^{H}}, \xi)$}
\end{defn}

Though we focus on a single instance problem setting (i.e., design for a unique goal specification for an initial state), one could easily envision settings where the robot may be required to carry out a number of different tasks, each corresponding to a different goal specification.
In such cases, the above definition can easily be extended to account for the multi-task nature of the setting. In particular, we can consider the max, min, or average of the $\MxGD$ and $\MnGD$ values across instances. The specific variation that may be used might depend on the nature of the setting. For example, if one were to create designs to account for the worst possible case across all instances, one would want to make sure that the max of $\MxGD$ and $\MnGD$ values across the instances fall below specific thresholds.

It is worth noting that the above definitions for the bounds consider a much larger set of plans than required. While the definitions consider all possible plans, humans may never consider most of those plans. For example, they might not think the robot would follow an extremely suboptimal plan.
This will result in weaker bounds, which could result in more extensive model updates than required. In the running example discussed above, regardless of where you place the hose, there will always be a plan where the robot could go fetch the hose and spray the plants. Thus, if one were to make changes to the model based purely on these metrics, only removing the hose completely from the setting would result in a setting where the use of the hose is not considered part of one of the possible outcomes. 


%% file: 5-1-compliation.tex
\section{Calculating $\MxGD$ and $\MnGD$}
The first order of business for us would be to calculate the approximations of $\GSD$, in particular, we will show how one could employ an off-the-shelf cost-optimal planner to calculate these values. The general idea we will employ is the fact that we will create a single planning problem which involves coming up with actions in the robot model and human model, and then finally having a set of check actions that check how the goal state is achieved under the human plan and model compared against the one achieved under the robot plan and model.

More formally, given the model pair $\mathcal{\lambda}= \langle \mathcal{M^R}, \mathcal{M^H}\rangle$, we create a new compiled model such that $\mathcal{M^\lambda}= \langle \mathcal{D^\lambda},  \mathcal{I^\lambda}, \mathcal{G^\lambda} \rangle$ where $\mathcal{D^\lambda}$ is the domain, defined by the tuple $\mathcal{D^\lambda} =  \langle \mathcal{F^\lambda}, \mathcal{A^\lambda} \rangle$. Here, $\mathcal{F^\lambda}$ is a set of fluents represented by $\mathcal{F^\lambda}= \mathcal{F^R} \cup \mathcal{F^H} \cup \mathcal{F^\theta} \cup \mathcal{F^\kappa}$, where $\mathcal{F^R}$ is the original set of fluents, and $\mathcal{F^H}$ is a copy of these fluents which correspond to the human's beliefs. We will use these copies to keep track of how the plan will unfold according to the human model and will use the notation $f_i^\mathcal{H}$ to denote the human copy of a fluent $f_i^\mathcal{R} \in \mathcal{F^R}$. 
$\mathcal{F^\theta}$ includes the housekeeping fluents ${robot\_can\_act}$ and ${human\_can\_act}$ which control when a human or robot can perform actions, and $\mathcal{F^\kappa}$ are a set of compare fluents. $\mathcal{F^\kappa}$ contains a compare fluent for every fluent in $\mathcal{F^R}$, i.e., $\exists f_i^\kappa \in \mathcal{F}^\kappa$, for every $f_i^\mathcal{R} \in \mathcal{F^R}$.
As discussed, our eventual objective is to compare the resultant goal states of the robot plan and a plan expected by the human. These compare fluents will allow us to track whether such comparisons have been performed.

$\mathcal{I^\lambda}$ is an initial state denoted by $\mathcal{I^\lambda} = \langle \mathcal{I^R} \cup \mathcal{I^H} \cup \{human\_can\_act\} \rangle$, where $\mathcal{I^R}$ is the robot's initial state, and $\mathcal{I^H}$ is a copy of this state, representative of the human's initial beliefs.
The inclusion of $\{human\_can\_act\}$, ensures that the plan can start with human actions.
$\mathcal{G^\lambda}$ is the set of goals shared by the human and robot, denoted by $\mathcal{G^\lambda} = \langle 
 \mathcal{G^R} \cup \mathcal{G^H} \cup \mathcal{F^\kappa} \rangle$. Here, $\mathcal{G^R} \subseteq \mathcal{F^R}$ is the goal specified in the original fluent set, and $\mathcal{G^H} \subseteq \mathcal{F^H}$ the same goal expressed in the human fluent copy, while $\mathcal{F^\kappa}$ are the original compare fluents, used to determine how similar the human and robot's final goal states are, once all goals have been achieved.

$\mathcal{A^\lambda}$ is a set of actions represented by $\mathcal{A^\lambda}= \langle \mathcal{A}^{\mathcal{R}'} \cup \mathcal{A}^{\mathcal{H}'} \cup \mathcal{A^\theta} \cup \mathcal{A^\kappa} \rangle$. 
Here, $\mathcal{A}^{\mathcal{R}'}$ is the action set corresponding to the robot actions $\mathcal{A^R}$. Here the action definitions are identical to their definitions in $\mathcal{A^R}$, except that for any $a \in \mathcal{A}^{\mathcal{R}'}$, you have ${robot\_can\_act} \in pre(a)$. Similarly, $\mathcal{A}^{\mathcal{H}'}$ is a copy of these actions corresponding to the human's beliefs of them (i.e., corresponds to their definitions in $\mathcal{A^H}$) expressed in $\mathcal{F^H}$. Additionally, for any $a \in \mathcal{A}^{\mathcal{H}'}$, you have ${human\_can\_act} \in pre(a)$. 

$\mathcal{A^\theta}$ are the special flip actions $a_{flip}^{\mathcal{R}}$ and $a_{flip}^{\mathcal{H}}$ which enable or disables a human or robot's ability to perform actions by changing the state of the $\mathcal{F^\theta}$ fluents. In our setting, the human begins with the ability to perform actions, while the robot does not. Once the human's goals have been achieved, their ability to perform actions is terminated, while the robot's are enabled. We define this specific human flip action, $\sharedPreP(a_{flip}^{\mathcal{H}})$, as follows:

\begin{itemize}
\item $ \sharedPreP(a_{flip}^{\mathcal{H}})/ \{human\_can\_act\} \subseteq
\mathcal{G^H}, \\ \sharedPreN(a_{flip}^{\mathcal{H}}) = \emptyset$:\\
  $\sharedAdd(a_{flip}^{\mathcal{H}}) = \{robot\_can\_act\}$,\\ $\sharedDel(a_{flip}^{\mathcal{H}}) = \{human\_can\_act\}$
\end{itemize}

Once the robot has achieved all of its goals, its ability to perform actions is disabled using the action $a_{flip}^{\mathcal{R}}$. We can define $a_{flip}^{\mathcal{R}}$ similar to the $a_{flip}^{\mathcal{H}}$.
These two actions ensure that the planner has identified a valid human and robot plan before performing all the check actions.



Once the human and robot have both executed their plans,
all fluent sets from their final goal states are compared, for which one check fluent action exists for each fluent in $\mathcal{F^R}$.  
$\mathcal{A^\kappa}$ is a set of compare actions that check for this consistency and is denoted by $\mathcal{A^\kappa} = A^\kappa_{f_1} \cup A^\kappa_{f_2} \cup A^\kappa_{f_3}...A^\kappa_{f_{|\mathcal{F}|}}$, 
such that the set $\mathcal{A}^\kappa_{f_i} = \{a^1_{f_i}, {a^2_{f_i}}, {a^3_{f_i}}, {a^4_{f_i}}\}$ exists for each $f_i^\mathcal{R} \in \mathcal{F^R}$.
We will call the first two copies the check disagreement actions for fact $f_i$ and the latter two the check agreement actions. The agreement copies will only fire if the human's belief about the fluent value matches the robot's, and the disagreement copy fires only in case they don't.
Additionally, we will modulate the cost parameters  $\mathcal{P}_1$ (agreement action cost) and $\mathcal{P}_2$ (disagreement action cost) to get different behaviors from the compilation. These are defined as follows:
\begin{itemize}
  \item $ \sharedPreP(a^1_{f_i}) = \{ f_i^\mathcal{R}\}, \\ \sharedPreN(a^1_{f_i}) = \{f_i^\mathcal{H}\} \cup \{robot\_can\_act\} \cup \{human\_can\_act\} \cup \{f_i^\kappa\}$,\\
  $\sharedAdd(a^1_{f_i}) = \{f_i^\kappa\}$,
  $\sharedDel(a^1_{f_i}) = \emptyset$, and $c(a^1_{f_i}) = \mathcal{P}_1$ 
  \item $ \sharedPreP(a^2_{f_i}) = \{ f_i^\mathcal{H}\}, \\ \sharedPreN(a^2_{f_i}) = \{f_i^\mathcal{R}\} \cup \{robot\_can\_act\} \cup \{human\_can\_act\} \cup \{f_i^\kappa\}$,\\
  $\sharedAdd(a^2_{f_i}) = \{f_i^\kappa\}$, 
  $\sharedDel(a^2_{f_i}) = \emptyset$, and $c(a^2_{f_i}) = \mathcal{P}_1$ 
  \item $ \sharedPreP(a^3_{f_i}) = \{ f_i^\mathcal{R}, f_i^\mathcal{H} \}, \\ \sharedPreN(a^3_{f_i}) = \{robot\_can\_act\} \cup \{human\_can\_act\}\cup \{f_i^\kappa\}$,\\
  $\sharedAdd(a^3_{f_i}) = \{f_i^\kappa\}$,  $\sharedDel(a^3_{f_i}) = \emptyset$, and $c(a^3_{f_i}) = \mathcal{P}_2$ 
  \item $ \sharedPreP(a^4_{f_i}) = \emptyset, \\
  \sharedPreN(a^4_{f_i}) = \{ f_i^\mathcal{R}, f_i^\mathcal{H}\}  \cup \{robot\_can\_act\} \cup \{human\_can\_act\}\cup\{f_i^\kappa\},$\\
  $\sharedAdd(a^4_{f_i}) = \{f_i^\kappa\}$,  $\sharedDel(a^4_{f_i}) = \emptyset$, and $c(a^4_{f_i}) = \mathcal{P}_2$ 
\end{itemize}

For a plan $\pi^{\lambda}$ that is valid for this new model $\mathcal{M^\lambda}$, we will use the notation $\mathcal{H}(\pi^{\lambda})$ to represent the sequence of human actions that appear in $\pi^{\lambda}$, and $\mathcal{R}(\pi^{\lambda})$ to represent the robot actions. Also, we will use the notation $\kappa^+(\pi^{\lambda})$ and $\kappa^-(\pi^{\lambda})$, to list the set of check agreement and check disagreement actions that appear in the plan.

One of the aspects of the model definition we haven't delved into is the action costs of the different actions. As we will see setting these costs to different values allows us to determine the values we are interested in. In general, we will assume that the cost of all actions in $\mathcal{A^\theta}$ are zero.

\begin{prop}
\label{prop:max}
    \textit{For a given compiled model $\mathcal{M^\lambda}$, let us set the action costs of all actions in $\mathcal{A}^{\mathcal{R}'} \cup \mathcal{A}^{\mathcal{H}'}$ to a unit cost, and set the disagreement cost as $\mathcal{P}_2 = 0$ and agreement cost as $\mathcal{P}_1 > 2^{|\mathcal{F^R}|+|\mathcal{F^H}|}$. For the given cost function, let $\pi^{\lambda}$ be an optimal plan, then $\MxGD(\mathcal{M^R}, \mathcal{M^H}) = |\kappa^-(\pi^{\lambda})|$.}
\end{prop}
\begin{proof}
    Now this comes from the fact that the cost of the plan is being dominated by check agreement actions. In particular, the cost of a single agreement action is higher than the combined cost of the longest possible plan in either the human or robot model (i.e., one that passes through each possible state). By setting the cost of agreement so high, we force the planner to select plans with a low degree of agreement.
\end{proof}

We can similarly calculate the $\MnGD(\mathcal{M^R, M^H})$, by inverting the costs, specifically:

\begin{prop}
    \textit{For a given compiled model $\mathcal{M^\lambda}$, let us set the action costs of all actions in $\mathcal{A}^{\mathcal{R}'} \cup \mathcal{A}^{\mathcal{H}'}$ to a unit cost, and set the agreement cost as $\mathcal{P}_1 = 0$ and disagreement cost as $\mathcal{P}_2 > 2^{|\mathcal{F^R}|+|\mathcal{F^H}|}$. 
    For the given cost function, let $\pi^{\lambda}$ be an optimal plan, then $\MnGD(\mathcal{M^R}, \mathcal{M^H}) = |\kappa^-(\pi^{\lambda})|$.}
\end{prop}
The proof is identical to the previous proposition.
\paragraph{Remark}One of the additional constraints we are placing on solutions to this problem is the requirement that human actions be performed before any robot actions. This is technically not a requirement for the validity of the compilation. If we had added $\{robot\_can\_act\}$ to the initial state, the compilation will allow for human and robot actions to be interleaved or picked in any order. We will refer to this version as the flattened version of the compilation. Flattening the compilation will allow for more solutions but at the cost of increasing the branching factor. One of the evaluations, we will perform is whether the two versions have any significant difference in computational characteristics.

\section{Identifying Minimal Designs for HRGA}
Now that we have methods for computing the $\MxGD$ and $\MnGD$ bounds, the question remains as to how to select the designs that will allow us to create models with the required properties for a given HRGA $\mathcal{DP}$.
In particular, we are interested in identifying designs that meet the requirements laid out in Definition \ref{def:kl}. However, rather than laying out the most general version, we will look at a specific instantiation of the definition that will allow us to use an even more efficient compilation than the version laid out in the previous section. In particular, we are interested in settings, where $\ell=0$ (the limit on the lower bound) and the set of design changes provided as input, correspond to adding or removing unique fluents from the human/robot initial states and each design has a unit cost. 

Here, we only allow initial state changes because, for most practical problem settings, initial state changes are the easiest changes the designer could make. From a theoretical point of view, one could always map changes to any other model component into an initial state change. For example, by making them conditioned on a static predicate, whose value is determined in the initial state.

The basic algorithm will have two loops, the outer loop will iteratively increase the allowed design cost. The inner loop will try to identify a design for the given budget constraint that will meet the requirement of $\MnGD$ of zero, and $\MxGD$ within some specified limit.



\input{5-1-algo}

\subsection{Inner Loop for Identifying Designs}
In the inner loop, we will follow a slightly modified version of $\MnGD$ to identify the design itself. For a set of possible designs $\mathbb{U}$, we can map each design to a specific addition or removal to the initial state for the human and robot model. Let $\tau$ be the current limit placed on the design size. The basic intuition here is that we will modify the $\MnGD$ compilation to first perform a set of actions corresponding to design changes. The design actions are disabled to calculate the human and robot plan that results in $\MnGD$ with $0$. We can directly encode this into the goal by looking for plans where all the fluent values match. We check whether the identified design allows the required $k$ bound on the $\MxGD$. If not, we look for another design with $\tau$ length which satisfies the $\MnGD=0$ requirement. We do this by updating the $\MnGD$ compilation to disallow previously identified designs.

We will extend our previous compiled model as follows, $\mathcal{M}^\lambda$ to $\mathcal{M}^\lambda_{\mathbb{U}} = \langle  \mathcal{D}^\lambda_{\mathbb{U}}, \mathcal{I}^\lambda_{\mathbb{U}}, \mathcal{G}^\lambda_{\mathbb{U}} \cup \{unseen\_design\}\rangle$ where $\mathcal{D}^\lambda_{\mathbb{U}}= \langle\mathcal{F}^\lambda_{\mathbb{U}}, \mathcal{A}^\lambda_{\mathbb{U}},c^{\lambda}_{\mathbb{U}}$$\rangle$. In this new model, we have $\mathcal{F}^\lambda_{\mathbb{U}} = \langle \mathcal{F^\lambda} \cup \{design\_allowed\} \cup \{unseen\_design\} \cup \mathcal{F^\tau} \cup\mathcal{F^{\tau+}} \cup \mathcal{F^{D}} \rangle$. Where $\{design\_allowed\}$ is used to keep track of when designs are allowed and $\{unseen\_design\}$ ensures that the current design used hasn't been used before. The fluent sets $\mathcal{F^\tau}$ and $\mathcal{F^{\tau+}}$ ensures only $\tau$ designs can be performed. Finally, $|\mathcal{F^{D}}|= |\mathbb{U}|$ keeps track of what exact designs were used. 

For the actions we have,
$\mathcal{A}^\lambda_{\mathbb{U}} = \langle \mathcal{A^\lambda}\cup \mathcal{A}^{\mathbb{U}} \setminus \mathcal{A^{\kappa-}}\cup \{design\_completed\} \rangle$. Where $|\mathcal{A}^{\mathbb{U}}| = \tau \times |\mathbb{U}| $, is the set of actions you have for the design, and $\mathcal{A^{\kappa-}}$ is the subset of check disagreement copies. Each design action updates the initial state per the design requirements. $\{design\_completed\}$ stops the design phase and allows the human and robot actions to be applied (from there on out, the actions are the same as the previous compilation). By not including the disagreement copy, we will look for robot/human plan pairs that can only satisfy the original check goal by using agreement actions (hence, the states need to match).

The new initial state is given as $\mathcal{I}^\lambda_\mathbb{U} = (\mathcal{I^\lambda}\setminus \{human\_can\_act\}) $$\cup \{unseen\_design, design\_allowed\} \cup \mathcal{F^\tau}$. Therefore, you can only start with design steps (which can be performed at most $k$ steps).
 
For the new design action, there exists an action for each possible design step (upper-bounded by $k$) and a design.
For a design related to fluent $f$ and step $i$, the positive precondition of the action would be $\sharedPreP(a) = \{design\_allowed, k_i\}$. If the design corresponds to making an initial state true, that fluent is part of the add effect, if makes a fluent false it becomes part of the delete effect. The action will always remove $t_i \in \mathcal{F^\tau}$ and add $t_i^+ \in \mathcal{F^{\tau+}}$ as well as the corresponding design. 
Now the goal is given as $\mathcal{G}^\lambda_{\mathbb{U}} = \langle \mathcal{G^\lambda} \cup \mathcal{F^{\tau+}} \rangle$. The $\{design\_completed\}$ action simply deletes $\{design\_allowed\}$ and adds the fluent $\{human\_can\_act\}$.
 Now the addition of $\mathcal{F^{\tau+}}$ means that $\tau$ design needs to be applied. The cost function is kept the same as $\mathcal{M^\lambda}$.
 A solution to this problem allows us to identify designs that result in zero $\MnGD$ automatically, and we can subsequently check $\MxGD$. If $\MxGD$ requirements are met, we know that this corresponds to the minimal cost design, and the solution is returned.
\input{6-1-table}

If this is not the case, we would want to disallow it and look for other designs of size $\tau$ that might suit our requirements. We will do this by introducing new conditional effects into the $\{design\_completed\}$ action, such that the condition for that effect corresponds to the design fluents of a previously identified design and the effect is to delete the $\{design\_completed\}$ fluent. 

Once the updated $\mathcal{M}^\lambda_{\mathbb{U}}$ no longer returns a solution, we know that no other minimal designs of that budget satisfy this requirement and the control is passed to the outer loop for checking a larger design budget. Algorithm \ref{algo1} provides a pseudo-code for this algorithm. $found\_designs$ is a set that is used to track all the previously found designs for the current design budget, and $all\_designs\_found$ is a flag that captures whether the algorithm has exhausted the space of all designs that can ensure a $\MnGD$ of zero.


%% file: 5-1-algo.tex
\begin{algorithm}[tbp!]
\footnotesize
\caption{An algorithm for a HRGA design problem}
\label{algo1}
\vspace{2pt} 
\begin{algorithmic}[1]
\STATE {\em Input:} $\mathcal{DP}, k$
\STATE {\em Output:} Model update set $\mathcal{U}\subseteq \mathbb{U}$ that satisfy the requirements that for resulting models $\MnGD$ is zero and $\MxGD$ is $k$


\FOR{$\tau$ in 1 ... $|\mathcal{F^R}|$ }

\STATE $all\_designs\_found \gets False$
\STATE $found\_designs \gets \{\}$

\WHILE{$all\_designs\_found$ is $false$}
    
    \STATE $ \mathcal{M}^{\lambda}_{\mathbb{U}}\gets \MnGD\_with\_Design(\mathcal{M^R}, \mathcal{M^H}, \tau$\\
~~~~~~~~~~~~~~~~~~~~~~~~~~~~~~~~~~~~~~~~~~~~~~~~~~$, found\_designs)$ 
    
    \STATE{$ \pi^\lambda_\mathbb{U} \gets GetPlan(\mathcal{M}^{\lambda}_{\mathbb{U}})$}
    
    \IF{$\pi^\lambda_\mathbb{U}$ length is $0$}
        \STATE $ all\_designs\_found\gets True $
    \ELSE
        \STATE Extract model updates $\mathcal{U}$ from $\pi^\lambda_\mathbb{U}$ and add it to $found\_designs$
        \IF{$\MxGD(\Lambda(\mathcal{M^R}, \mathcal{U}), \Lambda(\mathcal{M^H}, \mathcal{U}))$ is $k$}
        \RETURN  $\mathcal{U}$ 
            \ENDIF
        \ENDIF
\ENDWHILE
\ENDFOR
\end{algorithmic}
\end{algorithm}

%% file: 6-1-table.tex
\begin{table*}[ht]
\centering
\setlength\tabcolsep{3pt}
\begin{tabular}{|c|c|c|c|c|c|c|}
\hline
Domain &
  \texttt{Main} &
 \texttt{Main-fl} &
\texttt{Naive} &
\texttt{$\MnGD$} & 
\texttt{$\MnGD$ with Design} &
\texttt{$\MxGD$}\\ \hline
\multirow{5}{*}{Blocksworld}
& 73.849 $\pm$ 2.150 & 73.172 $\pm$ 2.281 & 387.178 $\pm$ 6.724 & 11.703 $\pm$ 1.264 & 12.955 $\pm$ 0.469 & 12.642 $\pm$ 1.461 \\
& 71.966 $\pm$ 3.183 & 74.115 $\pm$ 3.570 & 386.627 $\pm$ 4.365 & 11.674 $\pm$ 1.165 & 11.739 $\pm$ 2.044 & 12.893 $\pm$ 2.097 \\
& 111.919 $\pm$ 30.519 & 110.049 $\pm$ 27.237 & 432.541 $\pm$ 24.484 & 12.420 $\pm$ 5.547 & 11.883 $\pm$ 0.250 & 30.181 $\pm$ 9.809 \\
& 97.049 $\pm$ 1.961 & 96.051 $\pm$ 1.213 & 417.206 $\pm$ 3.636 & 11.942 $\pm$ 0.760 & 12.748 $\pm$ 0.595 & 35.035 $\pm$ 1.308 \\
& 118.487 $\pm$ 28.660 & 116.327 $\pm$ 30.162 & 453.887 $\pm$ 21.167 & 13.038 $\pm$ 6.836 & 12.505 $\pm$ 0.529 & 25.932 $\pm$ 12.534 \\
 \hline
\multirow{5}{*}{Depot}
& 35.593 $\pm$ 2.338 & 31.877 $\pm$ 2.929 & 188.556 $\pm$ 10.304 & 5.747 $\pm$ 1.517 & 5.234 $\pm$ 1.096 & 5.006 $\pm$ 0.465 \\
& 86.355 $\pm$ 0.810 & 85.794 $\pm$ 1.029 & 448.649 $\pm$ 4.008 & 13.530 $\pm$ 0.761 & 13.991 $\pm$ 0.293 & 16.285 $\pm$ 0.564 \\
& 84.901 $\pm$ 0.976 & 84.861 $\pm$ 1.396 & 446.248 $\pm$ 0.784 & 13.463 $\pm$ 0.633 & 14.149 $\pm$ 0.860 & 15.288 $\pm$ 0.604 \\
& 85.238 $\pm$ 1.170 & 85.853 $\pm$ 0.528 & 445.714 $\pm$ 2.721 & 13.455 $\pm$ 0.702 & 13.601 $\pm$ 0.368 & 15.479 $\pm$ 0.673 \\
& 86.531 $\pm$ 2.873 & 84.731 $\pm$ 1.181 & 452.672 $\pm$ 3.351 & 13.648 $\pm$ 0.668 & 14.428 $\pm$ 1.613 & 15.723 $\pm$ 0.205 \\ \hline
\multirow{5}{*}{Elevator}
& 3.691 $\pm$ 0.013 & 3.629 $\pm$ 0.009 & 17.930 $\pm$ 0.052 & 0.543 $\pm$ 0.008 & 0.607 $\pm$ 0.008 & 0.561 $\pm$ 0.001 \\
& 4.116 $\pm$ 0.013 & 4.070 $\pm$ 0.008 & 19.708 $\pm$ 0.018 & 0.597 $\pm$ 0.011 & 0.676 $\pm$ 0.011 & 0.610 $\pm$ 0.001 \\
& 4.119 $\pm$ 0.009 & 4.085 $\pm$ 0.020 & 19.771 $\pm$ 0.072 & 0.599 $\pm$ 0.011 & 0.674 $\pm$ 0.002 & 0.611 $\pm$ 0.002 \\
& 4.123 $\pm$ 0.011 & 4.066 $\pm$ 0.013 & 19.843 $\pm$ 0.065 & 0.601 $\pm$ 0.011 & 0.673 $\pm$ 0.002 & 0.615 $\pm$ 0.003 \\
& 4.118 $\pm$ 0.008 & 4.068 $\pm$ 0.006 & 19.796 $\pm$ 0.063 & 0.600 $\pm$ 0.010 & 0.672 $\pm$ 0.002 & 0.611 $\pm$ 0.001 \\ \hline
\multirow{5}{*}{Logistics}
& 33.062 $\pm$ 0.738 & 32.330 $\pm$ 0.337 & 50.306 $\pm$ 0.582 & 0.886 $\pm$ 2.536 & 0.807 $\pm$ 0.017 & 28.785 $\pm$ 0.755 \\
& 31.936 $\pm$ 0.495 & 30.381 $\pm$ 0.158 & 48.112 $\pm$ 0.643 & 0.684 $\pm$ 0.026 & 0.814 $\pm$ 0.016 & 27.577 $\pm$ 0.485 \\
& 30.457 $\pm$ 0.408 & 30.457 $\pm$ 0.209 & 47.927 $\pm$ 0.554 & 0.676 $\pm$ 0.023 & 0.804 $\pm$ 0.009 & 26.194 $\pm$ 0.393 \\
& 32.795 $\pm$ 0.488 & 32.824 $\pm$ 0.552 & 50.068 $\pm$ 0.469 & 0.684 $\pm$ 0.024 & 0.819 $\pm$ 0.018 & 28.455 $\pm$ 0.467 \\
& 30.439 $\pm$ 0.521 & 30.641 $\pm$ 0.414 & 48.040 $\pm$ 0.403 & 0.683 $\pm$ 0.023 & 0.806 $\pm$ 0.015 & 26.152 $\pm$ 0.469 \\ \hline
\multirow{5}{*}{Zenotravel}
& 5.084 $\pm$ 0.025 & 5.025 $\pm$ 0.017 & 23.641 $\pm$ 0.070 & 0.716 $\pm$ 0.009 & 0.791 $\pm$ 0.005 & 0.745 $\pm$ 0.005 \\
& 5.167 $\pm$ 0.024 & 5.142 $\pm$ 0.019 & 24.022 $\pm$ 0.157 & 0.728 $\pm$ 0.013 & 0.807 $\pm$ 0.002 & 0.755 $\pm$ 0.003 \\
& 7.025 $\pm$ 0.041 & 6.953 $\pm$ 0.045 & 31.263 $\pm$ 0.160 & 0.944 $\pm$ 0.016 & 1.044 $\pm$ 0.012 & 1.081 $\pm$ 0.004 \\
& 7.210 $\pm$ 0.066 & 7.164 $\pm$ 0.066 & 31.468 $\pm$ 0.126 & 0.949 $\pm$ 0.017 & 1.063 $\pm$ 0.009 & 1.148 $\pm$ 0.010 \\
& 10.021 $\pm$ 0.662 & 9.986 $\pm$ 0.636 & 40.853 $\pm$ 8.818 & 1.367 $\pm$ 0.026 & 1.496 $\pm$ 0.010 & 1.510 $\pm$ 0.015 \\
  \hline
  \end{tabular}
  \caption{The average and standard deviation time taken by each method compared to each baseline in seconds per instance. The first three columns respectively present the time taken by our method, a variation of our method that doesn't enforce ordering, and a baseline that iterates over possible designs. The final three columns report the average time taken to compute the lower bound of \GSD, lower bound with design, and upper bound.}
  \label{tab1}
  \end{table*}


%% file: 6-evaluation.tex
\section{Evaluation}
Our empirical evaluation objective was to provide a computational characterization of the different approaches to computing $\MxGD$ and $\MnGD$ measures and to perform designs that were introduced in this paper. 


\paragraph{Dataset} We looked at five standard IPC domains \cite{IPC2000, IPC2002}, and converted five problem instances from each into variations of a goal state divergence problem. To help minimize our problem run-time and potential planner issues, instances with smaller initial states were chosen. 

To create the human and robot models, these instances were first duplicated. For each instance, five problem variations containing human and robot models were then created. All original robot problem instances were kept the same as the original IPC problem instance. We created the human problem instance by deleting five random initial state fluents from the initial state of the original instance. This means five problem variations for each problem instance for each domain were created, for a total of 25 problem variations created per domain. All values listed in the table are averaged across these five randomly generated problem variations. This ensured that there was always a design set of size five within the required $\MxGD$ and $\MnGD$ limits. We also removed the zoom action from Zenotravel to avoid large variations in the fuel level fluents.






For consistency, we considered an $\MxGD$ limit of 0 as well. This allowed us to frame the calculation $\MxGD$ for the problem as checking whether the $\MxGD$ compilation is unsolvable if we force the plan to have at least one disagreement action. This allows us to perform cross-domain comparisons while keeping the $\MxGD$ constant and also avoid the use of costlier cost-optimal planners. 

To guarantee that we had problems with the required $\MxGD$ limit, we updated the goal specification of the problem instances selected from previous IPC competitions so all plans would result in the same goal state. 


\paragraph{Setup} 
We implemented the compilations for individually computing baselines of  $\MxGD$ and $\MnGD$ as well as the updated $\MnGD$ compilation that also identifies the design. 
We also implement a simple breadth-first search over the design space for the baseline. 

For design, our primary points of comparison will be our proposed algorithm (referred to as \texttt{Main}) and a naive one (\texttt{Naive}) that merely iterates over all possible designs and tests whether the designs result in zero upper and lower bounds. We will also consider a variation of the \texttt{Main} that considers the flattened compilation (\texttt{Main-fl}).

For each of these primary design algorithms (i.e., \texttt{Main}, \texttt{Main-fl}, and \texttt{Naive}), the time listed is the total time taken to find the minimal design that will ensure the upper and lower bounds are zero. As such, this
involves solving for upper and lower bounds multiple times.
Conversely, the times listed for $\MxGD$, $\MnGD$, and $\MnGD$ with design is the average time taken to compute each of these bounds individually (with ordering constraints enforced). To the best of our knowledge, we are the
first to tackle this problem, and we are unaware of any
existing baselines to compare this work against. Thus, we only consider baselines that provide the minimal design
for the target upper and lower bounds but with potentially different computational overheads (we have also provided a characterization of the hardness of calculating these bounds).

For each domain, we tested the three conditions on each instance, and we used Lama \cite{Richter_2010} for solving all compilations. All experiments were performed on a computer with an Apple M2 Max chip and 64 GB Ram. All experiments were run with a time limit of 60 minutes.

\paragraph{Results} Our primary metric is the time taken by each approach. Accordingly, Table \ref{tab1} presents the average and standard deviation time in seconds taken per each instance reported. 
Across all problems, we see that our \textbf{\texttt{Main} and \texttt{Main-fl} methods take a significantly much shorter time than the baseline}. For, \texttt{Main} and \texttt{Main-fl} were mostly comparable, with small variation between instances.
As such, we see that enforcing the ordering does provide a small improvement over the flattened compilation. 
Presumably, this is due to the fact that adding the additional structure would reduce the branching factor.
Also, compared with the naive baseline, which does not leverage planning to identify the design, takes a significantly shorter time. We note that the most noticeable benefit is for the Depot domain. We also note that the addition of design into $\MnGD$ compilation adds minimal overhead. Finally, the $\MxGD$ times are, in general, higher than $\MnGD$ times. However, this is expected since, in this setting, $\MxGD$ corresponds to testing for unsolvability.

%% file: 8-conclusion.tex
\section{Conclusion}
This paper presents the first attempt at developing a design framework to help align human expectations about how a goal specification may be achieved with the actual outcomes of a robot plan selected to satisfy the specification. Our focus in this paper has been to provide a clear framework to understand and study environment design within this context. As alluded to in the paper, the specific design problem we study is one among a number of different problems we could study in this space. In future work, we hope to explore some of these works and also look at studying these problems in the context of more complex decision-making frameworks. In particular, we would be interested in seeing how to adapt these mechanisms to support more complex objective/preference specification mechanisms including various forms of temporal logic and reward functions.